\documentclass[letterpaper, 10 pt, conference]{ieeeconf}  

\IEEEoverridecommandlockouts                              
\newtheorem{problem}{Problem}

\usepackage{graphicx}
\newcommand{\E}{\mathbf{E}}
\newcommand{\Pp}{\mathbf{P}}
\usepackage{xcolor}

\usepackage{amsmath}
\usepackage{amssymb}
\usepackage{algorithm}
\usepackage{algpseudocode}
\newtheorem{lemma}{Lemma}
\newtheorem{theorem}{Theorem}

\newcommand{\pe}{\pi^{\mathrm{E}}}
\newcommand{\hpe}{\hat\pi^{\mathrm{E}}}
\newcommand{\balpha}{\boldsymbol{\alpha_t}}
\newcommand{\R}{R}

\newcommand{\Real}{\mathbb{R}}

\newcommand{\policyest}{\hat{\Xi}}

\newcommand{\bv}{\textbf{b}}
\newcommand{\hRe}{\hat{\mathcal{R}}^\mathrm{E}}
\newcommand{\Id}{\mathbf{I}}
\newcommand{\ind}{\mathbb{I}}
\newcommand{\Exp}{\mathbb{E}}

\usepackage{algpseudocode}
\usepackage{subcaption}

\usepackage{graphicx}

\newcommand{\rev}[1]{{\color{black} #1}}

\title{\LARGE \bf
Efficient Reward Identification In Max Entropy Reinforcement Learning with Sparsity and Rank Priors}

\author{{\centering $\text{Mohamad Louai Shehab}^*$}
        \hspace*{0.2 in}
        { \centering $\text{Alperen Tercan}^*$}
            \hspace*{ 0.2 in}
        {\centering Necmiye Ozay}
\thanks{This work was supported in part by ONR CLEVR-AI MURI (\#N00014-21-1-2431). The authors are with the University of Michigan, Ann Arbor. Emails: $\{$mlshehab,tercan,necmiye$\}$@umich.edu
}
\thanks{$^*$These two authors contributed equally.}}

\begin{document}

\maketitle
\thispagestyle{empty}
\pagestyle{empty}

\begin{abstract}
In this paper, we consider the problem of recovering time-varying reward functions from either optimal policies or demonstrations coming from a max entropy reinforcement learning problem. This problem is highly ill-posed without additional assumptions on the underlying rewards. However, in many applications, the rewards are indeed parsimonious, and some prior information is available. We consider two such priors on the rewards: 1) rewards are mostly constant and they change infrequently, 2) rewards can be represented by a linear combination of a small number of feature functions. We first show that the reward identification problem with the former prior can be recast as a sparsification problem subject to linear constraints. Moreover, we give a polynomial-time algorithm that solves this sparsification problem exactly. Then, we show that identifying rewards representable with the minimum number of features can be recast as a rank minimization problem subject to linear constraints, for which convex relaxations of rank can be invoked. In both cases, these observations lead to efficient optimization-based reward identification algorithms. Several examples are given to demonstrate the accuracy of the recovered rewards as well as their generalizability.

\end{abstract}

\section{INTRODUCTION}
Reward identification, or Inverse Reinforcement Learning (IRL), is the problem of learning rewards from data. The premise behind IRL is that the reward function serves as the most succinct representation of an agent's behavior \cite{ng2000algorithms}. Learning a reward function from demonstrations allows agents to generalize beyond observed behaviors \cite{ng2000algorithms}, infer underlying human intentions \cite{sadigh2016information}, and capture pairwise preferences \cite{sadigh2017active,biyik2022learning}. However, like many inverse problems, IRL is inherently ill-posed as there may be infinitely many reward functions consistent with the same observed behavior \cite{ng1999policy,cao2021identifiability, kim2021reward}. For instance, suppose an agent moves from location 
$A$ to location 
$B$. One possible hypothesis is that the agent likes 
$B$, but another equally valid hypothesis is that the agent dislikes 
$A$. Both reward hypotheses are consistent with the observed behavior, highlighting the fundamental ambiguity in IRL \cite{skalse2023invariance}.

To address this ambiguity, recent research has shifted toward learning the entire \emph{set} of reward functions that can explain observed behaviors \cite{metelli2021provably,metelli2023towards,lindner2022active, metelli2024recent}. However, there is generally no consensus on how to select a specific reward function from within this set. The choice is often guided by the requirements of the downstream task or heuristic considerations. For example, Linear Programming IRL \cite{ng2000algorithms} and Max Margin IRL \cite{abbeel2004apprenticeship} select the reward function that makes the demonstrated policy as optimal as possible relative to the next-best alternative, effectively maximizing the opportunity cost—a principle widely studied in economics, where rational agents seek to maximize the value of their chosen actions relative to foregone alternatives \cite{buchanan1978cost}. Adversarial IRL \cite{fu2018learning} aims to find a reward function that generalizes well across environments, often selecting a state-only reward that maximizes transferability. Maximum Entropy IRL \cite{ziebart2008maximum} selects the reward function that maximizes the likelihood of the observed demonstrations, assuming an entropy-regularized policy model. More recently, \cite{lazzati2025partial} introduced a framework for quantitatively selecting the best reward—potentially outside the solution set—based on a given target application.

In this work, we study reward identification in finite-horizon settings with time-varying reward functions. This is an important generalization for real-world applications where rewards evolve over time due to changing preferences, environmental conditions, or task requirements. While most prior IRL methods assume static rewards, the dynamic setting makes IRL even more ill-posed, as reward ambiguity can now arise at every time step \cite{cao2021identifiability, shehab2024learning}. Existing approaches for dynamic rewards either (1) impose restrictive parametric assumptions (e.g., Gaussian random walks \cite{ashwood2022dynamic} or generalized linear models \cite{nguyen2015inverse}), limiting their expressiveness to predefined reward dynamics, or (2) assume privileged knowledge on the number of underlying reward regimes \cite{likmeta2021dealing}. Relatedly, learning time-varying objective functions is also considered in the area of inverse optimal control \cite{rickenbach2023time, westermann2020inverse}. 

Building on prior work, we propose a principled framework that systematically incorporates structure-aware priors—such as minimal reward switches and shared feature bases—to resolve ambiguity in time-varying reward identification, enabling flexible yet interpretable reward identification without strong parametric assumptions. Our contributions include (1) a polynomial-time algorithm for recovering minimally switching rewards, (2) a convex relaxation for feature-based reward decomposition, and (3) robustness guarantees under finite-sample policy estimates. Empirical results validate our approach in several gridworld environments, showing improved interpretability and transferability over existing methods.

\section{PRELIMINARIES}
\subsection{Notation}
$\mathbb{R}$ and $\mathbb{N}$ are the sets of real and natural numbers respectively. The identity matrix in $\mathbb{R}^{n\times n}$ is denoted by $\Id_n$. The zero matrix in $\mathbb{R}^{m\times n}$ is denoted by $\mathbf{0}_{m\times n}$ ($m$ are $n$ are dropped sometimes when they are clear from context). $\mathbf{1}_m$ is the constant vector of ones in $\mathbb{R}^m$. $\mathbb{I}(x\in X)$ is the indicator function. Given a matrix $A$, $\mathrm{rank}(A) \text{ and }\mathrm{colspan}(A)$ denote its rank and column span, respectively. When $B$ is another matrix of compatible dimension, $\begin{bmatrix}
    A & B
\end{bmatrix}$ denotes the horizontal concatenation of $A$ and $B$, and $A\otimes B$ denotes their Kronecker product. Given a vector space $V$ with a basis $B = \{v_1, \cdots,v_m\}$, $[w]_B$ is the vector representation of $w$ in $V$. For a set $S$, $\Delta(S)$ denotes the set of probability distributions over it, and $|S|$ denotes its cardinality. 

\subsection{Markov Decision Processes}
A Markov Decision Process (MDP) is a tuple $\mathcal{M} = (\mathcal{S}, \mathcal{A}, \mathcal{\mathcal{T}},\mu_0, r,\gamma,T)$, where $\mathcal{S}=\{s^{(1)},\dots,s^{(n)}\}$ is a finite set of states with cardinality $|\mathcal{S}|= n$; $\mathcal{A}=\{a^{(1)},\dots,a^{(m)}\}$ is a finite set of actions with cardinality $|\mathcal{A}|= m$; $\mathcal{T}: \mathcal{S} \times \mathcal{A} \to \Delta(\mathcal{S})$ is a Markov transition kernel; $\mu_0 \in \Delta(S)$ is an initial distribution over the set of states; $r = (r_t)_{t=0}^{T-1} \text{ is a time-varying reward function where each } r_t : \mathcal{S} \times \mathcal{A} \to \mathbb{R}$ is the reward function at time step $t$; $\gamma \in [0,1]$ is a discount factor; and $T \in\mathbb{N}$ is the non-negative time horizon. An MDP without a reward function, denoted $\mathcal{M}\setminus r$,  is called an \emph{MDP model}.  A policy $\pi_t:\mathcal{S}\to \Delta(\mathcal{A})$ is a function that describes an agent’s behavior at time step $t$ by specifying an action distribution at each state. We denote by $\pi = (\pi_t)_{t=0}^{T-1}$ the \emph{time-varying} stochastic policy throughout the entire horizon. A trajectory $\tau$ (of length $T$) is an alternating sequence of states and actions (ending with a state), i.e., $\tau = (s_0,a_0 ,s_1,a_1, \dots, s_{T-1},a_{T-1},s_T)$ with $s_t\in\mathcal{S}$ and $a_t\in\mathcal{A}$. Under a policy $\pi$, a trajectory $\tau$ occurs with probability
\begin{equation}
    \mathbb{P}_{\mu_0}^\pi(\tau) = \mu_0(s_0) \prod \limits_{t=0}^{T-1} \pi_t(a_t|s_t) \prod\limits_{t=0}^{T-1} \mathcal{T}(s_{t+1}|s_t,a_t),
\end{equation}
which depends on the distribution of initial states, the policy, and the Markov transition kernel. We consider the Maximum Entropy Reinforcement Learning (MaxEntRL) objective given by:
\begin{equation}\label{eq:max_ent_obj}
    J_{\text{MaxEnt}}(\pi;r) = \mathbb{E}^{\pi}_{\mu_0}[\sum\limits_{t=0}^{T-1} \gamma^t \biggl(  r (s_t,a_t) + \mathcal{H}(\pi_t(.|s_t)) \biggr)],
 \end{equation}
where $\mathcal{H}(\pi_t(.|s_t)) = -\sum\limits_{a\in \mathcal{A}} \pi_t(a|s_t)\log(\pi_t(a|s_t))$ is the entropy of the policy $\pi_t$. The expectation is with respect to the probability measure $\mathbb{P}^\pi_{\mu_0}$. 
We define the optimal policy $\pi_r^*$, corresponding to a reward function $r$, as the maximizer of (\ref{eq:max_ent_obj}), i.e.:
\begin{equation}\label{eq:opt_prob}
\pi_r^* = \arg \max\limits_{\pi} J_{\text{MaxEnt}}(\pi;r),
\end{equation}
which is known to be unique \cite{geist2019theory} up-to accessible states. The maximum entropy policy is \cite{ziebart2010modeling, haarnoja2017reinforcement,gleave2022primer}:
\begin{equation}\label{eq:soft_policy}
    \pi_t^*(a|s) =\frac{e^{Q_t^*(s,a)}}{\sum\limits_{a' \in \mathcal{A}} e^{Q_t^*(s,a')}}
\end{equation}
where $Q^*_t$ is the optimal soft Q-function at time step $t$, given by the following backward-in-time computation:
\begin{align}
Q^*_{T-1}(s,a) &= r_{T-1}(s,a), \notag\\
Q_t^*(s, a) &= r_t(s,a) +  \notag \\
& \gamma \mathbb{E}_{s' \sim P(.|s,a)} [ \underbrace{\log(\sum\limits_{a'\in\mathcal{A}}\exp(Q^*_t(s',a')))}_{\triangleq V^*_{t+1}(s')}],  
\end{align}
with $ s \in \mathcal{S}$, $a \in \mathcal{A}$, for  $t < T-1 $. $V^*_t$ is the optimal soft value function at time step $t$, which is also known as reward-to-go.

\subsection{Inverse Reinforcement Learning}\label{sec:prem:irl}
Inverse reinforcement learning is the problem of inferring a reward function given an agent's actions \cite{ng2000algorithms}. Concretely, given an MDP model $\mathcal{M}\setminus r$ and an expert's time-varying policy $\pe$, the goal is to find a reward function $r$ such that $\pe$ is the optimal policy for $r$, in other words $r$ \emph{induces} $\pe$. However, this problem is ill-posed because multiple distinct reward functions can yield the same optimal policy, making reward inference inherently ambiguous \cite{ng1999policy,cao2021identifiability, kim2021reward,skalse2023invariance}. In the case of the Max Entropy RL objective, the set of reward functions that induce a given policy 
$\pe$ can be derived in closed form \cite{cao2021identifiability,shehab2024learning}. We present the following result.
\begin{lemma}[\cite{cao2021identifiability}]\label{lem:cao}
For any time-varying policy $\bar \pi_t(a|s):\{0,\dots,T-1\}\times\mathcal{A}\times\mathcal{S}\rightarrow (0,1]$, and for any function $\nu:\{0,\dots,T\}\times\mathcal{S}\rightarrow\mathbb{R}$, the reward function given by
\begin{align}\label{eq:cao}
r_t(s,a) =  \log \bar \pi_{t}(a|s) - \gamma \mathbb{E}_{s'}[ \nu_{t+1}(s')] +  \nu_{t}(s),
\end{align}
with $\nu_T=0$, is the only reward function for which $\bar{\pi}$ is the optimal solution of \eqref{eq:opt_prob} with optimal soft value function $V^*_t = \nu_t,$ for all $t$.
\end{lemma}

Similar to \cite{shehab2024learning}, we vectorize Equation~\eqref{eq:cao} to define the set of rewards consistent with a given policy. To this end, we define the matrix $\Phi_T \in \mathbb{R}^{Tmn \times Tn}$ and the vector $\Xi^{\mathrm{E}}\in \mathbb{R}^{Tmn}$ as: \\
\begin{equation}\label{eq:Psi_and_Xi}
    \Phi_T =  \left[\begin{array}{c c c c c}
    -\E & \gamma \Pp & \mathbf{0}  & \cdots & \mathbf{0} \\
    \mathbf{0} & -\E & \gamma \Pp  & \cdots & \mathbf{0} \\
   \vdots &  \vdots &  \ddots & \vdots &  \vdots  \\
   \vdots &  \vdots &  \ddots & -\E &  \gamma \Pp  \\
    \mathbf{0} & \cdots &\cdots  &  \mathbf{0} & -\E  \\
\end{array}\right], \Xi^{\mathrm{E}} = \begin{bmatrix}
         \pi_{0}^{\texttt{log}} \\
         \pi_{1}^{\texttt{log}} \\
        \vdots \\
          \pi_{T-1}^{\texttt{log}}
    \end{bmatrix},
\end{equation}
with $\mathbf{E} = \mathbf{1}_m \otimes \Id_n $, $\mathbf{P} = \begin{bmatrix}
    P_{a^{(1)}}^\intercal & \cdots & P_{a^{(m)}}^\intercal
\end{bmatrix}^\intercal \in \mathbb{R}^{mn\times n}$, where $P_{a^{(k)}} \in \mathbb{R}^{n\times n}$ is such that its $ij$-th entry is given by $\mathcal{T}(s^{(j)}|s^{(i)}, a^{(k)})$, $k \in \{1,\dots,m\}$ and $\pi_t^{\texttt{log}}$ is the vectorized policy given by:
\begin{equation*}
    \pi_{t}^{\texttt{log}} =  \begin{bmatrix}
          \log(  \pi^E_t(a_1|s_1)) \\
          \log(  \pi^E_t(a_1|s_2))\\
        \vdots \\
         \log(  \pi^E_t(a_m|s_{n}))
    \end{bmatrix} \in\mathbb{R}^{mn}, \quad t = 0,1,\dots, T-1.
\end{equation*}
The subscript $T$ in   $\Phi_T$ emphasizes the number of block rows.
Letting $r = [r_0^\top, \cdots,r_{T-1}^\top]^\top$ and $\nu = [\nu_0^\top, \cdots, \nu_{T-1}^\top]^{\top}$, the set of rewards and value functions inducing $\pe$ can be compactly written as:
\begin{equation} \label{eq:feasible-set}
    \mathcal{R}^{\mathrm{E}} = \left\{\begin{bmatrix}
        r \\\nu
    \end{bmatrix} \mid \begin{bmatrix} \Id_{Tmn} \; \Phi_T  \end{bmatrix} \begin{bmatrix}
        r \\\nu
    \end{bmatrix} = \Xi^{\mathrm{E}}\right\}.
\end{equation}
 
\section{PROBLEM STATEMENTS}\label{sec:problem-statements}
By looking at the definition of $\mathcal{R}^{\mathrm{E}}$, it should be clear from a simple dimension argument that there exists an infinite number of rewards inducing the same policy $\pe$. Indeed, the key insight from Lemma~\ref{lem:cao} is that for any expert policy 
$\pe$, one can generate a valid inducing reward $r$ by 
selecting an appropriate time-dependent value function $\nu$ 
and computing $r$ via Equation~\eqref{eq:cao}. Hence, recovering any reward function that induces $\pe$ is not particularly meaningful. Instead, we are interested in recovering reward functions consistent with prior knowledge we might have about the structure of the reward function. One such prior is to find the time-varying reward function inducing $\pe$ while having the minimum number of switches. This is particularly important in many real-world settings, where less erratic reward functions enhance interpretability and better reflect underlying task structures (Occam's razor). This gives rise to the first problem considered in this paper, formally given as:
\begin{problem}\label{prob:p1}
    Given an MDP model $\mathcal{M}\setminus r$ and a time-varying policy $\pe$, find the reward function $r$ inducing $\pe$ with the least number of switches, i.e. $r_t = r_{t+1}$ for as many $t$'s as possible. 
\end{problem}

Another approach is to find a reward function that induces $\pi^\mathrm{E}$, expressed in terms of a structured basis, commonly referred to in the IRL literature as \textit{featurization} \cite{abbeel2004apprenticeship, memarian2020active, bloem2014infinite}. Specifically, the reward function at time $t$ is given by:  
\begin{equation}
    r_t(s,a) = \sum\limits_{k=1}^K \alpha_{k,t} u_k(s,a) = U \boldsymbol{\alpha_t},
\end{equation}
where $U \in \mathbb{R}^{mn\times K}$ is a feature matrix, with each column corresponding to a feature function $u_k$, and $\boldsymbol{\alpha_t} = [\alpha_{1,t}, \dots, \alpha_{K,t}]^\intercal$ represents the corresponding feature weights. However, unlike standard IRL settings where the feature matrix $U$ is typically predefined, in our case, \textbf{$U$ is unknown}. This leads to the second problem we address:  

\begin{problem}  \label{prob:p2}
    Given an MDP model $\mathcal{M} \setminus r$ and an expert policy $\pi^\mathrm{E}$, determine a candidate feature matrix $U$ and weight vectors $\{\boldsymbol{\alpha_t}\}$ such that the reward function  
    $r_t = U \boldsymbol{\alpha_t}$ induces $\pi^\mathrm{E}$.  
\end{problem}

Both problems \ref{prob:p1} and \ref{prob:p2} can be extended to the setting when only a finite-sample estimate $\hat{\pi}^{\text{E}}$ of the expert policy is available, rather than the true policy $\pi^{\text{E}}$. In Section~\ref{sec:fse}, we address this practical scenario by introducing robust variants of $\mathcal{R}^{\text{E}}$ that accounts for estimation errors and provides probabilistic guarantees.

\section{METHODOLOGY}
Throughout this section, our solutions will be different instantiations of the following optimization problem:
\begin{equation}\label{eq:template}\tag{P}
    \begin{aligned}
    \min_{r,\nu} \quad & \ell(r) \\
    \text{s.t.} \quad & \begin{bmatrix}
        r\\\nu 
    \end{bmatrix} \in \mathcal{R}^{\mathrm{E}}. \\
\end{aligned}
\end{equation}
 for some loss function $\ell: \{r_0,\cdots, r_{T-1}\} \to \mathbb{R}$. Problem~\eqref{eq:template} serves as our unifying optimization framework, where domain-specific knowledge is systematically incorporated through tailored loss functions $\ell$, while the constraint  ensures consistency with the expert policy.
\subsection{Minimally Switching Rewards}
Problem 1 can be reformulated naturally as a \emph{sparsification problem}, where the objective is to maximize the number of zero entries in an appropriately defined vector-valued sequence. In particular, we define the differences $\Delta r_t$ between the rewards at every two consecutive time steps:
\begin{equation}
    \Delta r_t = r_{t+1} - r_t,\; t = 0, \cdots, T-2,
\end{equation}
and consider the sequence $\{\Delta r_t\}_{t=0}^{T-2}$. 
It should be clear that any non-zero element $\Delta r_t$ corresponds to a switch in the reward function $r$. Hence, to minimize the number of switches, a natural optimization problem is the following:
\begin{equation}\tag{P1}\label{eq:main_opt}
    \begin{aligned}
    \min_{r,\nu} \quad & \|\{\Delta r_t\}_{t=0}^{T-2}\|_0 \\
    \text{s.t.} \quad & \begin{bmatrix}
        r\\\nu 
    \end{bmatrix} \in \mathcal{R}^{\mathrm{E}}, \\
    \quad & \Delta r_t = r_{t+1} - r_t,\; t = 0,\cdots, T-2,
\end{aligned}
\end{equation}
where $\|\{\Delta r_t\}_{t=0}^{T-2}\|_0\triangleq |\{t\mid \|\Delta r_t\|\neq 0\}|$.
While maximizing sparsity is a non-convex and hard to solve problem in general \cite{amaldi1998approximability,davis1997adaptive}, there exist efficient convex relaxations based on variants of $\ell_1$-norm \cite{ozay2011sparsification}. Moreover, as we show next, thanks to the additional structure in $\mathcal{R}^{\mathrm{E}}$, problem~\eqref{eq:main_opt} admits an \emph{exact} polynomial-time solution. 

In what follows, we devise a greedy algorithm to solve problem~\eqref{eq:main_opt} and prove its correctness. To do so, we
define the parametric truncated counterpart of Equation~\eqref{eq:feasible-set} with time-invariant rewards as:
 \begin{equation}\label{eq:truncated-feasible}
\mathcal{R}_{i:j}^{\mathrm{inv}}(\nu_o) =   \left\{\begin{bmatrix}
        r \\\nu
    \end{bmatrix}  \mid \left[\arraycolsep=1.4pt \begin{array}{c|c|c}
        \mathbf{\bar E} & \Phi_{j-i} & \begin{array}{c}
            \textbf{0} \\
            \gamma \Pp
        \end{array} 
    \end{array}\right] \begin{bmatrix} 
        r \\\nu \\ \nu_o
    \end{bmatrix} = \Xi^{\mathrm{E}}_{i:j}\right\}
\end{equation}
%
where $\mathbf{\bar E} = \mathbf{1}_{m} \otimes \Id_{mn}$ and
$\Xi^{\mathrm{E}}_{i:j}=[ {\pi_{i}^{\texttt{log}}}^\intercal,\ldots, {\pi_{j-1}^{\texttt{log}}}^\intercal]^\intercal  \in \Real^{(j-i)mn}$.

Our strategy is to find intervals over which a time-invariant reward can explain the given policy while ensuring consistency with the overall policy. In particular, we do this
by working backward in time and iteratively extending an interval until a time-invariant reward is not feasible over this interval and starting a new interval from that point. Algorithm~\ref{alg:greedy_partitioning} implements this idea by following a bisection approach to find the time step where the time-invariant reward becomes infeasible. 

\begin{algorithm}
\caption{Greedy Interval Partitioning}\label{alg:greedy_partitioning}
\begin{algorithmic}[1]

\Require Horizon: $T$, Transition Matrix: $\mathbf{P}$, Policy $\pi$
\Ensure Sequence of minimum number of switch times $Z$
and corresponding reward functions $\R$
\State $Z \gets ()$, $\R \gets ()$, $V_t \gets 0 \; \forall t : 0 \leq t\leq T$

\State $l \gets -1$, $u \gets T$, $j \gets T-1$, $\tau \gets T$
\While{$j \geq 0$}
    
    \If{$\mathcal{R}_{j:\tau}^{\mathrm{inv}}(V_\tau) \neq \emptyset$}
        \State $u \gets j$, 
        \State Pick $\bar{r}, \bar{\nu}$ from $\mathcal{R}_{j:\tau}^{\mathrm{inv}}(V_\tau)$
    \Else
        \State $l \gets j$
        \If{$u = l +1$}
            \State Prepend $u$ to $Z$ and $\bar{r}$ to $\R$
            \State $V_t \gets\bar{\nu}_t \;\; \forall t \in [u, \tau-1]$
            \State $\tau \gets u$, $l \gets -1$
        \EndIf
    \EndIf
    \State $j \gets \lfloor(l+u)/2\rfloor$
\EndWhile 
\State Prepend $\bar{r}$ to $\R$ and $V_t \gets\bar{\nu}_t \;\; \forall t \in [0, \tau-1]$\\
\Return {$Z$, $\R$}
\end{algorithmic}
\end{algorithm}

\begin{theorem}\label{thm:greedy-optimal}
    Algorithm~\ref{alg:greedy_partitioning} returns an optimal solution to Problem~\eqref{eq:main_opt}. Moreover, it has polynomial-time complexity. 
\end{theorem}

Before proving the theorem, we explain some of the notation in the algorithm.
By convention, the set $\mathcal{R}^{\mathrm{inv}}_{-1:j}(\nu)$ is considered to be empty for all $j$ and $\nu$. Moreover, $V_t$ denotes reward-to-go at time $t$, i.e., it is equal to $\nu_t$. The variables $l$ and $u$ are auxiliary bounds used in the bisection procedure, representing the current lower and upper bounds for the switch time being searched. Finally, $\tau$ is the horizon of the current interval for which the switch time is being searched.
The following two lemmas will be useful in proving this theorem.

\vspace{2pt}
\begin{lemma}\label{lem:non-extendable}
    Let the switch times returned by Algorithm~\ref{alg:greedy_partitioning} be $Z= [t_{k}, t_{k-1}, \dots, t_{1}]$ with $t_i > t_{i+1}$ and define $t_0 = T$. There is not any feasible time-invariant reward function for interval $t \in [t_{i+1} - 1, t_{i}-1]$ for any $i \geq 0$, i.e., $\mathcal{R}_{t_{i+1} - 1:t_{i}}^{\mathrm{inv}}(\nu_o)$ is empty for all $\nu_o \in \Real^n$. 
\end{lemma}

\begin{proof}
First, we note that
per Lines~9 and 10 of the algorithm, $u$ is added as the new switch time only after $\mathcal{R}_{u-1:\tau}^{\mathrm{inv}}(V_\tau)$ is found empty. Then,
for all $i$, $\mathcal{R}_{t_{i+1} - 1:t_{i}}^{\mathrm{inv}}(V_{t_{i}})$ is empty. Next, we will show that this is equivalent to $\mathcal{R}_{t_{i+1} - 1:t_{i}}^{\mathrm{inv}}(\nu_o)$ being empty for all $\nu_o \in \Real^n$. 

It is trivial that if $\mathcal{R}_{t_{i+1} - 1:t_{i}}^{\mathrm{inv}}(\nu_o)$ is empty for all $\nu_o \in \Real^n$, $\mathcal{R}_{t_{i+1} - 1:t_{i}}^{\mathrm{inv}}(V_{t_{i}})$ is empty as well. \rev{For the other direction, assume there exists $\bar{\nu}_o$ such that $\mathcal{R}_{t_{i+1} - 1:t_{i}}^{\mathrm{inv}}(\bar{\nu}_o)$ is not empty} and pick $(r,\nu) \in \mathcal{R}_{t_{i+1} - 1:t_{i}}^{\mathrm{inv}}(\bar{\nu}_o)$. Define $r'$ as $r_{T-1}'(s,a) = r_{T-1}(s,a) + \gamma \mathbb{E}_{s'\sim P(.|s,a)}\rev{[ \bar{\nu}_o(s') - V_{t_i}(s')]}$ and $r'_t(s,a) = r_t(s,a)$ for all $t < T-1$ and for all $a$ and $s$. Then, it can be seen by inspection that $(r',\nu)$ is in $\mathcal{R}_{t_{i+1} - 1:t_{i}}^{\mathrm{inv}}\rev{(V_{t_i})}$. Hence, $\mathcal{R}_{t_{i+1} - 1:t_{i}}^{\mathrm{inv}}\rev{(V_{t_i})}$ is not empty.
\end{proof}

\vspace{2pt}
\begin{lemma}\label{lem:greedy-ordered}
Let the switch times returned by Algorithm~\ref{alg:greedy_partitioning} be $[t_{k}, t_{k-1}, \dots, t_{1}]$. For any feasible switch time sequence $[t'_{k'}, t'_{k'-1}, \dots, t'_{1}]$:
$$
t_{i} \leq t'_{i} \qquad \forall i:  1 \leq i \leq \min(k,k')
$$ 
\end{lemma} 
\vspace{2pt}
\begin{proof}
We will show this by induction on $i$. At each step, we will use contradiction to show that $t'_i$ cannot be less than $t_i$.

\underline{Base case} ($i = 1$): Assume $t'_1 < t_1$; hence, $t'_1 \leq t_1-1$. Then, by Lemma~\ref{lem:non-extendable}, $[t'_1, T-1]$ cannot have a feasible time-invariant reward. Thus, $t'_1$ must be greater than or equal to $t_1$.

\underline{Induction Step:} Assume that $t'_i \geq t_i$ and $t'_{i+1} < t_{i+1}$. Then, $[t_{i+1}-1, t_i-1] \subset [t'_{i+1}-1, t'_i-1]$. Hence, the time-invariant reward that is feasible for $[t'_{i+1}-1, t'_i-1]$ is feasible for $[t_{i+1}-1, t_i-1]$ as well. This contradicts Lemma~\ref{lem:non-extendable}.
\end{proof}
\vspace{2pt}

Now we are ready to prove Theorem~\ref{thm:greedy-optimal}.

\begin{proof} (of Theorem~\ref{thm:greedy-optimal})
Assume that Algorithm~\ref{alg:greedy_partitioning} returns a sequence of reward functions $\R = [\bar r_{k+1}, \bar r_{k}, \dots, \bar r_1]$ and switch times  $Z =[t_{k}, t_{k-1}, \dots, t_{1}]$. 
We start our proof by showing that the returned rewards form a feasible solution to Problem~\eqref{eq:main_opt}. 

Let $\nu$ be a reward-to-go function such that for all $1 \leq i \leq k+1$, $[\bar{r}_i, \nu_{t_{i}}, \nu_{t_i+1}, \dots, \nu_{t_{i-1}-1}]$ is in $\mathcal{R}_{t_i:t_{i-1}}^{\mathrm{inv}}(\nu_{t_{i-1}})$, where $t_0$ and $t_{k+1}$ are taken as $T$ and $0$, respectively. It can be seen by inspection that $V$ populated on Line~11 of Algorithm~\ref{alg:greedy_partitioning} is such a $\nu$. Now, define $r_t$ to be $\bar r_i$ for all $t_i \leq t < t_{i-1}$ for all $i$. Then, $[r_0,\dots,r_{T-1}, \nu_0, \dots, \nu_{T-1}]$ is  in $\mathcal{R}^\mathrm{E}$. Hence, Algorithm~\ref{alg:greedy_partitioning} yields a feasible solution.

Now, by contradiction, assume that Algorithm~\ref{alg:greedy_partitioning} returns a feasible but suboptimal solution. Then, there exists a feasible solution with switch times $[t'_{k'}, t'_{k'-1}, \dots t'_1]$ with $k' < k$. If $[0, t'_{k'}]$ is a valid interval that can be solved with a time-invariant reward, $\mathcal{R}_{0:t'_{k'}}^{\mathrm{inv}}(\nu_o)$ is not empty for some $\nu_o$. However, by Lemma~\ref{lem:non-extendable}, $\mathcal{R}_{t_{k'+1}-1:t_{k'}}^{\mathrm{inv}}(\nu_o)$ is empty for all $\nu_o$. Since $[{t_{k'+1}-1:t_{k'}}] \subseteq [0, t'_{k'}]$ by Lemma~\ref{lem:greedy-ordered}, this is a contradiction. Therefore, Algorithm~\ref{alg:greedy_partitioning} finds an optimal solution. 

Finally, the runtime of the algorithm is primarily dominated by the operation in Line 6, which involves solving a system of equalities with $(m + \tau - j)n$ variables and $mn(\tau - j)$ constraints, which can be done in polynomial-time. This operation is performed $O(k \log T)$ times. Therefore, the overall algorithm runs in polynomial-time.
\end{proof}

We also remark that when the reward function is known to be featurized—that is, all $r_t$ can be expressed as weighted sums of common feature functions—the problem can be further simplified by restricting the search to the space of feature weights, as described in \cite{shehab2024learning}. However, this requires the knowledge of the feature functions. In the next section, we show how this requirement can be avoided.


\subsection{Feature-Based Rewards}

If the reward function at time step $t$ is expressed as $r_t = U \balpha$, optimizing over both $U$ and $\balpha$ leads to a bilinear optimization problem. Further, since the number of features is unknown, the dimension of $U$ is an additional decision variable.  Bilinear programs are generally NP-hard due to their inherent non-convexity, and even checking local optimality can be computationally intractable \cite{konno1976cutting,al1983jointly}. Our key insight to avoid solving a bilinear program is that featurization imposes a low-rank structure on the reward function, which remains consistent across the entire horizon. This means that while the reward at each time step may vary, it lies in a subspace spanned by a fixed set of basis functions. Consequently, instead of independently optimizing \( U \) and \( \boldsymbol{\alpha_t} \), we can directly model the reward function as a low-rank matrix, where each column corresponds to the reward at a given time step. By enforcing a low-rank structure on this matrix, we transform the problem into one of recovering a structured representation of rewards rather than solving a bilinear optimization.  Consequently, our objective is:

\begin{equation}\tag{P2}\label{eq:p2}
    \begin{aligned}
    \min_{r,\nu} \quad & \mathrm{rank}(\begin{bmatrix}
        r_0 \cdots r_{T-1}
    \end{bmatrix}) \quad
    \text{s.t.} \quad \begin{bmatrix}
        r \\ \nu 
    \end{bmatrix} \in \mathcal{R}^{\mathrm{E}}. \\
\end{aligned}
\end{equation}
The feature matrix and weights can then be recovered from the optimal solution of \eqref{eq:p2} as follows:
\begin{align}\label{eq:recovered}
    U &= \mathrm{colspan}(\begin{bmatrix}
        r_0 \cdots r_{T-1}
    \end{bmatrix}), \notag \\
    \boldsymbol{\alpha_t} &= [r_t]_U, \quad t = 0, \cdots, T-1.
\end{align}
While Problem~\eqref{eq:p2} is still a difficult non-convex problem \cite{vandenberghe1996semidefinite}, several heuristics have been developed to handle it, e.g., see \cite{beran1996combined,skelton2013unified}. Notably, it has been established in \cite{recht2010guaranteed}  that the nuclear norm, under some regularity assumptions, serves as the tightest convex approximation for the rank function, generalizing $\ell_1$-norm based relaxation of the $\ell_o$-quasinorm to the rank function. Thus, instead of Problem~\eqref{eq:p2}, we solve:
\begin{equation}\tag{P2-approx}\label{eq:p2-approx}
    \begin{aligned}
    \min_{r,\nu} \quad &\|\begin{bmatrix}
        r_0 \cdots r_{T-1}
    \end{bmatrix} \|_{*} 
    \quad  \text{s.t.} \quad  \begin{bmatrix}
        r \\ \nu 
    \end{bmatrix} \in \mathcal{R}^{\mathrm{E}}. \\
\end{aligned}
\end{equation}
where for a given matrix $A\in \mathbb{R}^{m\times n}$, its nuclear norm $\|A\|_* = \sum_{i=1}^{\min(m,n)} \sigma_i(A)$ with $\sigma_i$ denoting the singular values of $A$.
To get a better approximation of the rank function, nuclear norm relaxation can be further refined by considering an iterative reweighted variant \cite{mohan2010reweighted}. While we tried this variant in our experiments, the results remained identical to those obtained with the nuclear norm formulation, which was sufficiently accurate for our problem.
\section{Given Demonstrations}\label{sec:fse}

The set $\mathcal{R}^{\mathrm{E}}$ depends on the expert policy through $\Xi^\mathrm{E}$. However, in practice, the true expert policy $\pe$ is typically unknown and a finite-sample estimate $\hpe$ must be used instead. We use the following lemma to motivate our approach in this case.



\begin{lemma}\label{lem:from-dem}
Fix a timestep $t$, state $s \in \mathcal{S}$, and number of samples $n(t, s) \in \mathbb{N}$. Assume there exists a function $\alpha_t(s, \cdot): \mathcal{A} \to \mathbb{R}$ such that $\alpha_t(s,a) \leq \pe_t(a | s)$ for all $a \in \mathcal{A}$.
Let $\{a_i\}_{i=1}^{n(t, s)}$ be a collection of actions independently drawn from the policy $\pe_t(\cdot | s)$. 
Define the empirical estimate $\hpe_t(a | s)$ as:
\begin{equation}\label{eq:pi-hat-def}
    \hpe_t(a | s) = \frac{1}{n(t, s)} \sum_{i=1}^{n(t, s)} \mathbb{I}[a_i = a].
\end{equation}

Given a confidence level $\delta \in (0, 1)$,  the following inequality holds for all actions $a \in \mathcal{A}$:
\begin{equation} \label{eq:log-dev-bound}
    \mathbb{P}\left(\left| \log \hpe_t(a | s) - \log \pe_t(a | s) \right| 
    \leq \frac{\epsilon(t, s)}{\alpha_t(s,a) - \epsilon(t, s)}\right) \geq \delta.
\end{equation}
where:
\begin{equation} \label{eq:eps-defin}
    \epsilon(t, s) \triangleq \sqrt{\frac{1}{2n(t, s)} \log\left(\frac{2}{1 - \delta}\right)}.
\end{equation}
\end{lemma}

\begin{proof}
Consider a fixed action $a$ and let $\{X_1, X_2, \dots, X_{n(t,s)}\}$ be random variables defined such that $X_i = \ind[a_i = a]$ for all $i$. Note that $\hpe_t(a|s) = \frac{1}{n(t,s)}\sum_{i=1}^{n(t,s)} X_i$ and $\Exp[\hpe_t(a|s)] = \pe_t(a|s)$. Then  by Hoeffding's inequality, for any $\epsilon>0$ we have:
\begin{equation*}
    \mathbb{P}\left(|\hpe_t(a|s) - \pe_t(a|s)| \leq \epsilon \right) \geq 1 - 2e^{-2n(t,s)\epsilon^2}.
\end{equation*}
%
Picking $\epsilon = \epsilon(t,s)$ as defined in Equation~\eqref{eq:eps-defin} yields:
\begin{equation}\label{eq:prob-dev-bound}
    \mathbb{P}\left(|\hpe_t(a|s) - \pe_t(a|s)| \leq \epsilon(t,s)\right) \geq \delta.
\end{equation}

To go from Equation~\eqref{eq:prob-dev-bound} to Equation~\eqref{eq:log-dev-bound}, we need a lower bound for both $\hpe_t(a|s)$ and $\pe_t(a|s)$. To this end, define the events $\mathcal{E}_1 \triangleq \{ |\hpe_t(a|s)  - \pe_t(a|s)| \leq \epsilon(t,s) \}$ and $\mathcal{E}_2 \triangleq \{\min(\hpe_t(a|s), \pe_t(a|s))) \geq \alpha_t(s,a) - \epsilon(t,s)\}$. 

Since $\alpha_t(s,a) \leq \pe_t(a|s)$,  we have $\mathcal{E}_1 \subseteq \mathcal{E}_2$. Thus 
\begin{equation}\label{eq:double-prob}
\mathbb{P}(\mathcal{E}_1, \mathcal{E}_2) = \mathbb{P}(\mathcal{E}_1)\geq \delta.
\end{equation}
Define $\mathcal{E}_3 \triangleq \{     |\log\hpe_t(a|s) - \log\pe_t(a|s)| \leq \frac{\epsilon(t,s)}{\alpha(t,s) - \epsilon(t,s)}\}$, which is the event of interest in Inequality~\eqref{eq:log-dev-bound}. By Mean Value Theorem we have that for some $\xi \geq 0$ between $\hpe_t(a|s)$ and $\pe_t(a|s)$ the following holds:
\begin{align*}
    \log\hpe_t(a|s) - \log\pe_t(a|s)  = \frac{1}{\xi} (\hpe_t(a|s) - \pe_t(a|s))
\end{align*}

This implies that:
\begin{equation*}
|\log\hpe_t(a|s) - \log\pe_t(a|s)| \leq\frac{|\hpe_t(a|s) - \pe_t(a|s)|}{\min(\hpe_t(a|s), \pe_t(a|s)))},
\end{equation*}
from which we see that $\mathcal{E}_1\cap \mathcal{E}_2 \subseteq \mathcal{E}_3$. Therefore, by Equation~\eqref{eq:double-prob}, we get $\mathbb{P}(\mathcal{E}_3) \geq \delta$, which is the desired result.
\end{proof}

In practice, the data is given in the form of a set of trajectories $\mathfrak{D} = \{\tau_i\}_{i=1}^N$ with $\tau^i = (s_0^i,a_0^i, s_1^i,\dots,a_{T-1}^i, s_T^i)$.
For a given state $s$ and time $t$, we define the number of samples as $n(t,s) \triangleq \sum_{i=1}^N \ind[s_t^i = s]$ and set of actions sampled from $\pi_t(\cdot|s)$ as $\{a_t^i \mid  s_t^i = s \}$. Note that $n(t,s)$ is independent of the sampled actions as actions sampled at time $t$ only impacts future states. Then, we construct a sample consistent estimate  of the true policy $\pe$ by computing the relative frequency of actions at each state for each $t$ from  $0$ to $T-1$ as shown in Equation~\eqref{eq:pi-hat-def}.

Note that, while using maximum entropy policies guarantees that $\pi_t$ is always positive, a lower bound function $\alpha_t(s,a)$ may not be known. In practice, we replace $\alpha_t(s,a) -\epsilon(t,s)$ in Equation~\eqref{eq:log-dev-bound} with $\hpe_t(a|s) - \epsilon(t,s)$. When $n(t,s)$ is large enough, this gives a good lower bound to both $\hpe_t(a|s)$ and $\pe_t(a|s)$.

While increasing the number of samples reduces estimation error as can be seen in Equation~\eqref{eq:eps-defin}, even small inaccuracies can disrupt important structural properties, such as sparsity or low-rank characteristics, inherent in the true reward function. To address this issue, we introduce a policy-noise-robust variant of $\mathcal{R}^{\mathrm{E}}$, ensuring that solutions remain feasible for at least one possible realization of $\Xi^{\mathrm{E}}$. \rev{
Specifically, given a set of trajectories $\mathfrak{D}$ and a confidence level $\delta$, we estimate the policy using Equation~\eqref{eq:pi-hat-def} and obtain the estimated policy vector $\policyest^{\mathrm{E}}$ as in Equation~\eqref{eq:Psi_and_Xi}. We also construct an error bound vector $\bv$ from the upper bound in Equation~\eqref{eq:log-dev-bound} as follows.
}
\rev{
Define the error vector \(\boldsymbol{\epsilon} \triangleq [\boldsymbol{\epsilon}_0^\top, \boldsymbol{\epsilon}_1^\top, \dots, \boldsymbol{\epsilon}_{T-1}^\top]^\top \in \mathbb{R}^{Tmn}\), where for each \(t \in \{0, \dots, T-1\}\), we have \(\boldsymbol{\epsilon}_t \triangleq \mathbf{1}_m \otimes [\epsilon(t,s_1), \dots, \epsilon(t,s_n)]^\top \).
%
Then, \(\mathbf{b} \in \mathbb{R}^{Tmn}_{\ge 0}\) is given by $\mathbf{b} = \boldsymbol{\epsilon} \oslash 
\left( \exp\left( \hat{\boldsymbol{\Xi}}^{\mathrm{E}} \right) 
- \boldsymbol{\epsilon} \right)$,
where \(\exp(\cdot)\) and \(\oslash\) denote elementwise exponentiation and division, respectively.}

Finally, we define the robust reward set:
\begin{equation}
    \hRe = \left\{
    \begin{bmatrix}
        r \\\nu
    \end{bmatrix}
    \mid \policyest^{\mathrm{E}} - \bv \leq  \begin{bmatrix} \Id_{Tmn} \; \Phi_T  \end{bmatrix}
    \begin{bmatrix}
        r \\\nu
    \end{bmatrix} \leq \policyest^{\mathrm{E}} + \bv
    \right\}.
\end{equation}



In our numerical experiments when we only have access to demonstrations, we replace $\mathcal{R}^{\mathrm{E}}$ with $\hRe$ in the corresponding optimization problems.
\section{EXPERIMENTS}

The goal of our experiments is to answer two questions:
\begin{enumerate}
    \item[Q1.] How well can our frameworks recover the ground-truth time-varying rewards?
    \item[Q2.] Can our recovered rewards transfer to novel environments?
\end{enumerate}

To answer Q1, we evaluate Problems~\eqref{eq:main_opt} and \eqref{eq:p2-approx} in the 5x5 gridworld shown in Figure~\ref{fig:open_gridworld}. Each cell of the gridworld represents a state of the MDP. The actions available for the agent in each state are: $\{up,down,left, right,stay\}$. Upon taking an action, the agent transitions to the desired cell with a probability $1-p_{w}$, and transitions to a neighboring cell in one of the cardinal directions with a probability $p_w$, representing the wind probability. The MDP has two important landmarks: a home state (called $s_{\mathrm{home}}$) at the top left, and a water state (called $s_{\mathrm{water}}$) in the bottom middle. For example, an agent with high reward at the home state tries to reach the home as fast as possible. An agent with a uniform reward everywhere tries to explore the environment equally. By varying the rewards over time, we can capture and model a multitude of complex behaviors. For example, the agent might want to explore the environment at first. During exploration it gets ``thirsty", and thus the reward switches to reach the water state as fast as possible. Eventually, the agent wants to go back to the home state. A horizon of $50$ timesteps is used in all our experiments. Our frameworks improve over other baselines in qualitatively recovering the ground-truth weights and feature functions.

To answer Q2, we evaluate our frameworks in a transfer learning setting, where the reward function is learned in the gridworld of Figure~\ref{fig:open_gridworld}, but optimized in a different gridworld with different dynamics, shown in Figures~\ref{fig:blocked_gridworld} and \ref{fig:sticky_gridworld}. The difference in the dynamics is characterized by adding blockings, shown as a dashed line, and adding ``sticky" states, shown in yellow, where all actions result in staying in that state with a probability $0.8$. We show that rewards learned with our algorithms still produce optimal or near-optimal behaviors, while baseline methods produce either lower quality policies or rewards that generalize poorly.

\begin{figure}
    \centering
    \begin{subfigure}[b]{0.3\linewidth}
        \centering
        \includegraphics[width=\linewidth]{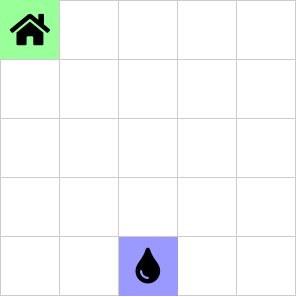}
        \caption{Open}
        \label{fig:open_gridworld}
    \end{subfigure}
    \hfill
    \begin{subfigure}[b]{0.3\linewidth}
        \centering
        \includegraphics[width=\linewidth]{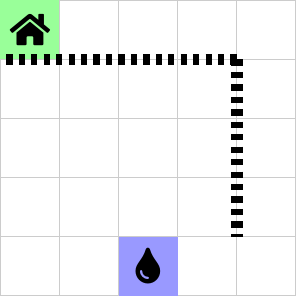}
        \caption{Blocked}
        \label{fig:blocked_gridworld}
    \end{subfigure}
    \hfill
     \begin{subfigure}[b]{0.3\linewidth}
        \centering
        \includegraphics[width=\linewidth]{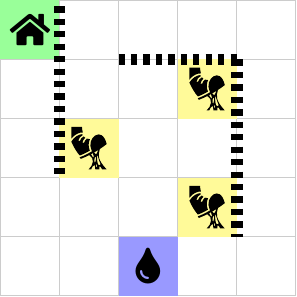}
        \caption{Sticky}
        \label{fig:sticky_gridworld}
    \end{subfigure}
    \caption{\small 5x5 Gridworlds used in the experiments.}
    \label{fig:gridworlds}
    \vspace{-0.3cm}
\end{figure}

\subsection{Experiment 1: Minimally Switching Rewards}

In this experiment, we evaluate the performance of Algorithm~\ref{alg:greedy_partitioning}. We construct tasks by dividing the time horizon into $k+1$ intervals, where $k = 5$ switch points are chosen uniformly at random. These switch points correspond to time steps where $\Delta r_t \neq 0$. We begin with a time-invariant reward function, with values sampled uniformly between 0 and 1, for the first interval. For each interval, we generate a time-invariant reward function by perturbing the previous interval’s reward using a uniformly sampled perturbation from $[0,\beta]^{mn}$. We vary $\beta$ from 0.1 to 0.4 for each subsequent interval to induce reward changes of increasing magnitudes.

To evaluate the accuracy of the inferred switch times, we consider the resulting intervals as clusters.
\rev{The predicted intervals are evaluated against the true interval partitioning using the Adjusted Rand Index (ARI), as defined in Equation 5 of \cite{hubert1985comparing}. The ARI is a widely used measure for comparing two clusterings: it equals 1 when they agree perfectly and has an expected value of 0 under random labeling (with possible negative values if agreement is worse than chance).}


To study the impact of using a finite-sample estimate of the policy $\pe$, we compute the ARI scores of the switch times identified by Algorithm~\ref{alg:greedy_partitioning} when run on estimates of $\pe$ obtained from varying numbers of trajectories. The experiment is repeated with 10 different reward functions. We report the mean and standard deviation of the ARI scores for each setting in Table~\ref{tab:ari_results}. As seen in Table~\ref{tab:ari_results}, increasing the number of trajectories leads to higher ARI scores and reduced variance. Our algorithm eventually recovers the true switching times. \rev{It is important to observe that the number of switches our algorithm identifies is always less than or equal to the true number of switches, hence our algorithm is able to explain the data with a simpler reward model when there is more uncertainty in the low-data regime. }

\begin{table}[h]
\centering
\begin{tabular}{|c|c|c|}
\hline
\textbf{Number of Trajectories} & \textbf{ARI} & \textbf{$\#$ of Switches} \\
\hline
True Policy & 1.000 ± 0.000 & 5.0 ± 0.0 \\
8{,}000{,}000 & 1.000 ± 0.000 & 5.0 ± 0.0 \\
4{,}000{,}000 & 0.968 ± 0.040 & 4.8 ± 0.4 \\
2{,}000{,}000 & 0.881 ± 0.118 & 4.1 ± 0.3 \\
1{,}000{,}000 & 0.726 ± 0.100 & 3.5 ± 0.5 \\
800{,}000  & 0.674 ± 0.179 & 3.0 ± 0.45 \\
400{,}000  & 0.566 ± 0.196 & 2.0 ± 0.63 \\
200{,}000  & 0.393 ± 0.098 & 1.1 ± 0.3 \\
\hline
\end{tabular}
\caption{\small \rev{Mean and standard deviation of ARI and the number of switches found across different dataset sizes over 10 random reward functions. 
The row labeled ``True Policy" corresponds to results obtained using the exact policy $\pe$ instead of trajectory data. Confidence level $\delta$ is set to $0.9999$ when defining the robust reward set.}}
\label{tab:ari_results}
\vspace{-0.33cm}
\end{table}

\subsection{Experiment 2: Feature-Based Rewards}

For this experiment, we generated two ground-truth feature functions $u_1,u_2:\mathcal{S}\times \mathcal{A}\to\mathbb{R}$ representing the home state and water state positions. The reward function is given by:
\begin{equation*}
    r_t^{\mathrm{true}}(s,a) = \alpha_{1,t}u_1(s,a) + \alpha_{2,t}u_2(s,a) 
\end{equation*}
where $u_1(s,.) = 1$ if $s =  s_{\mathrm{home}}$, and $0$ otherwise. Similarly, $u_2(s,.) = 1$ if $s =  s_{\mathrm{water}}$, and $0$ otherwise. We generate the time-varying weights $\boldsymbol{\alpha_t}$ following a Gaussian random walk as in \cite{ashwood2022dynamic}. After finding the expert policy $\pe$, we solve \eqref{eq:p2-approx} to find both the feature functions and the weights. The recovered time-varying weights are shown in Figure~\ref{fig:exp3}, which also includes results with policies esimated from demonstrations. The recovered feature functions\footnote{Since the feature vectors are $|\mathcal{S}|\times |\mathcal{A}|$ dimensional vectors, we only show the first $|\mathcal{S}|$ components, which correspond to the first action. The plots are the same for the remaining actions.} are shown in Figures~\ref{fig:home-state-label} and \ref{fig:water-state-label} . As a benchmark, we implemented the dynamic IRL method from \cite{ashwood2022dynamic}. We note that the number of demonstrations fed to the method from \cite{ashwood2022dynamic} is much fewer than what our method used due to scalability issues with the former (yet we used 5 times the number of trajectories reported in \cite{ashwood2022dynamic}). Since the reward decomposition in \eqref{eq:recovered} is not unique, there exists infinitely many valid choices of basis vectors $U$, leading to different recovered parameters for each basis choice. Thus, to qualitatively compare the ground-truth and recovered weights, and generate meaningful visualizations, we apply a two-step post-processing approach.  First, we identify a basis $U$ that satisfies the condition in \eqref{eq:recovered}. We then perform a projection step to align this basis with the ground-truth feature vectors, yielding a transformed basis $U'$. Next, we express the recovered weights relative to $B'$ and apply a standardization step to eliminate trivial invariances due to shifting and scaling. This ensures that the recovered weights are comparable to the ground-truth while preserving their relative structure. Figures~\ref{fig:home-state-label} and \ref{fig:water-state-label} compare our recovered feature mapping with that of \cite{ashwood2022dynamic}. By enforcing a low-rank decomposition in the objective function, our method successfully recovers the true feature functions, whereas dynamic IRL \cite{ashwood2022dynamic} produces a reasonable but noisier approximation. 

\begin{figure}[h!]
    \centering
    \includegraphics[width=0.9\linewidth]{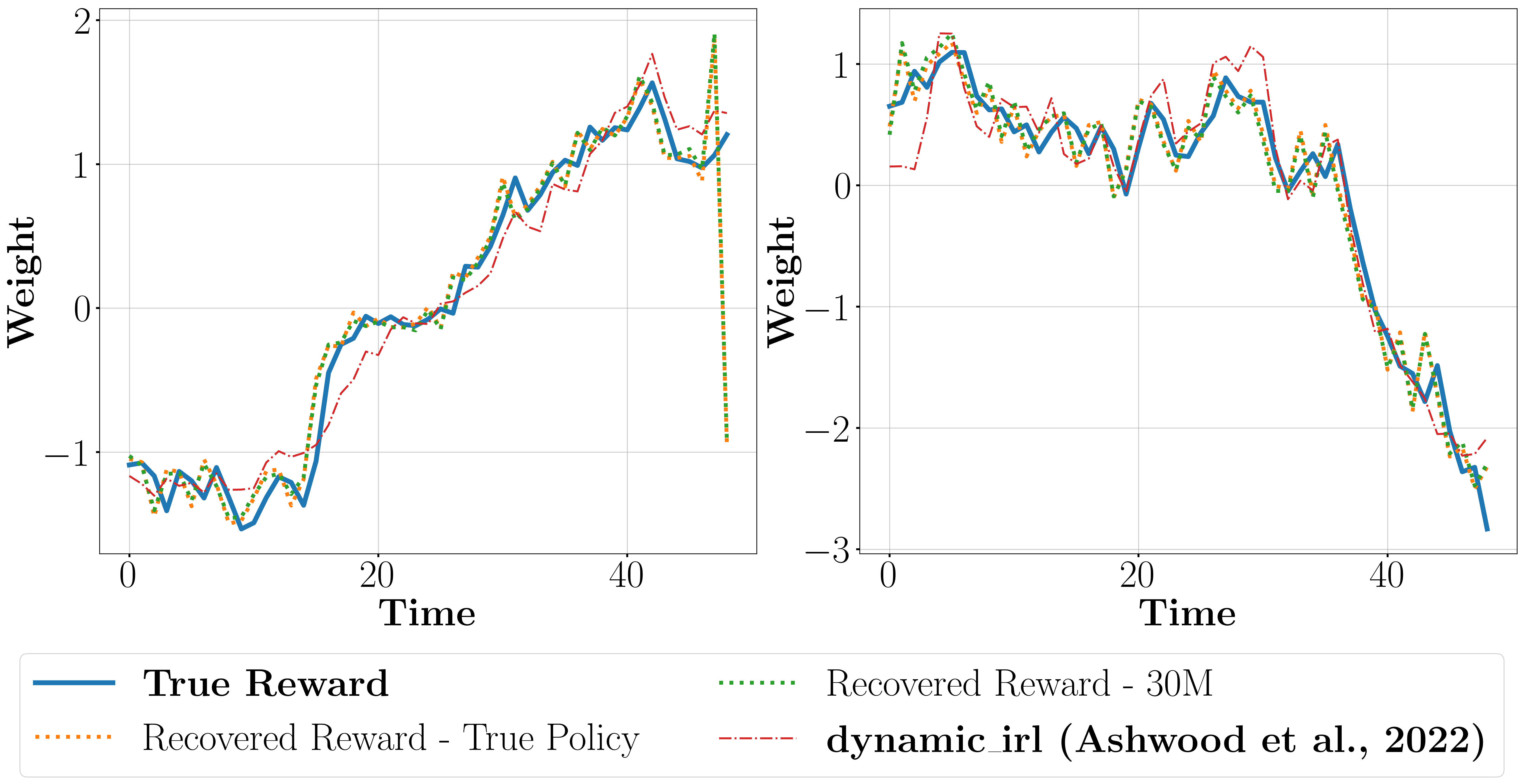}
    \caption{\small Recovered weights for Experiment 2.}
    \label{fig:exp3}
    \vspace{-0.3cm}
\end{figure}

\begin{figure}[h!]
    \centering
    \includegraphics[width=0.9\linewidth]{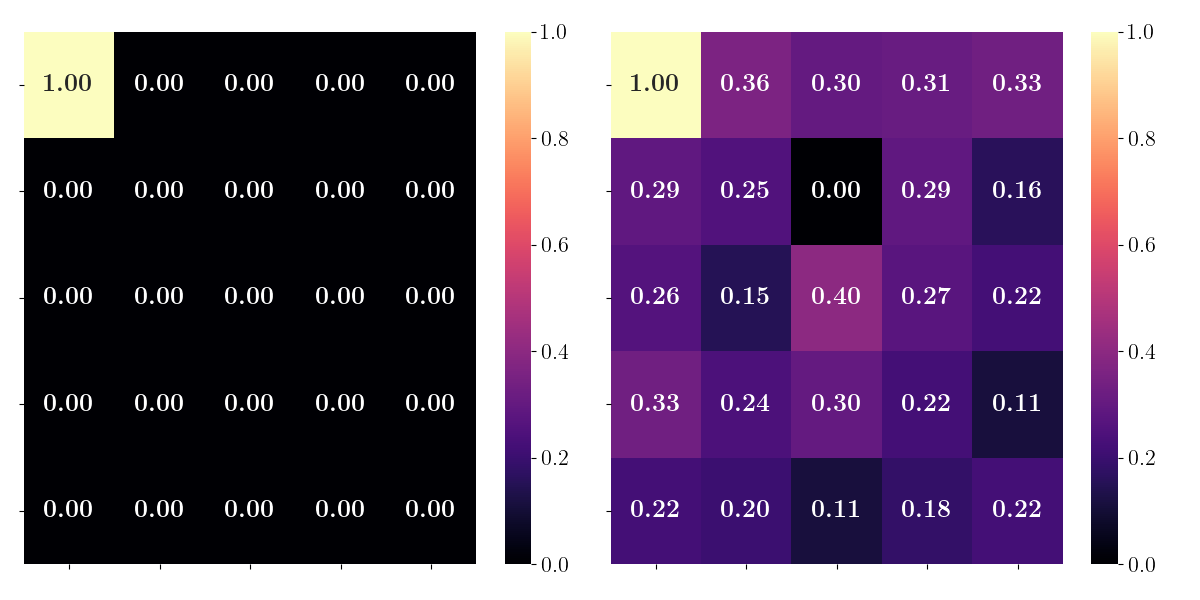}
    \caption{\small Recovered feature vector for the home state: (left) Ours vs. (right) \cite{ashwood2022dynamic}.}
    \label{fig:home-state-label}
    \vspace{-0.3cm}
\end{figure}

\begin{figure}[h!]
    \centering
    \includegraphics[width=0.93\linewidth]{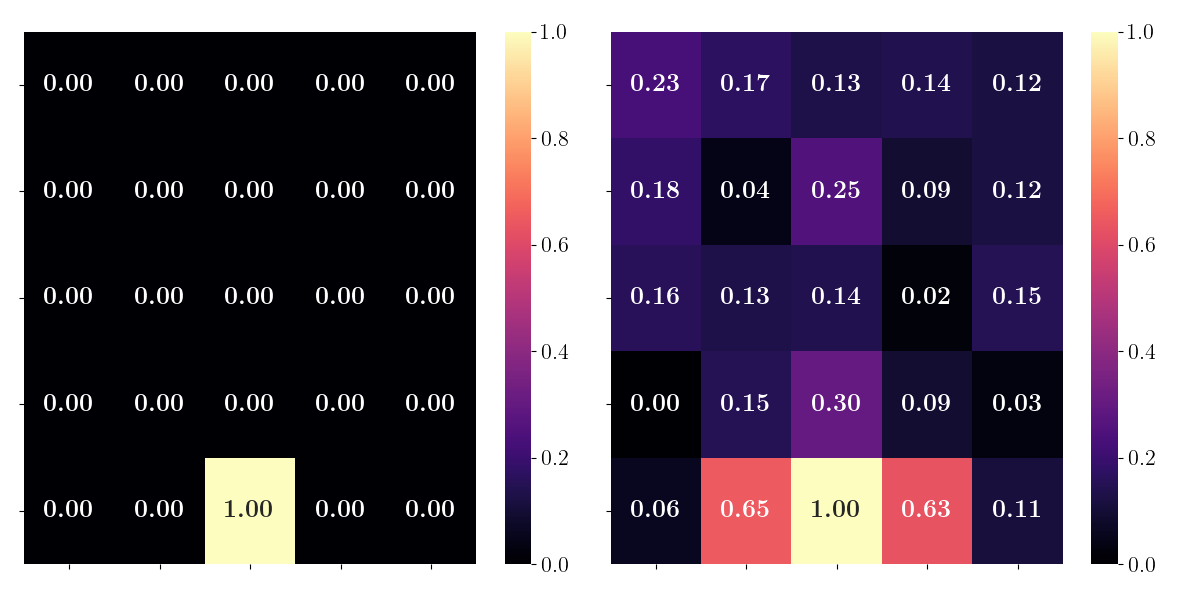}
    \caption{\small Recovered feature vector for the water state: (left) Ours vs. (right) \cite{ashwood2022dynamic}.}
    \label{fig:water-state-label}
    \vspace{-0.33cm}
\end{figure}

For our transferability experiments, we implemented an additional baseline: the finite-horizon MaxEnt IRL of \cite{ziebart2008maximum}, which recovers a time-invariant reward. To assess transferability, we first solve for the reward function using different methods/baselines with the optimal ground-truth policy in the {Open Gridworld} as input. We then compute the optimal policies of these recovered rewards in the novel environments, namely the Blocked Gridworld and the Sticky Gridworld. We report the negative log-likelihood of a sample trajectory set, generated from the optimal policy for the ground-truth reward, for each of the computed policies. We report these log-likelihoods in Table~\ref{tab:nll_table}.
Our algorithm achieves near-optimal performance in both novel environments, attaining the best transferability performance among all methods. It is worth mentioning that the Gaussian random walk structure of the weights is embedded in the learning algorithm of \cite{ashwood2022dynamic}, which we do not assume in our approach. Also, finite-horizon MaxEnt IRL baseline is given the true feature function.  Finally, both \cite{ashwood2022dynamic} and \cite{ziebart2008maximum} produce state-only reward functions, which are usually more suited for transferability tasks \cite{fu2018learning}.

\begin{table}[h!]
\renewcommand{\arraystretch}{1.3} 
    \centering
    \begin{tabular}{|c|c|c|}
    \hline
         Policy & Blocked Gridworld & Sticky Gridworld   \\ \hline
        $\pi_1^*$ &$-1.2851$ & $-1.3021$\\ \hline
        $\hat\pi_1^*$ & $-1.2850$& $-1.3019$ \\ \hline
        $\pi_{\mathrm{true}}^*$ & $\mathbf{-1.2495}$ & $\mathbf{-1.2223}$\\ \hline
        $\pi_{r}^*$ (Ours) & $\mathbf{-1.2516}$ & $\mathbf{-1.2438}$\\ \hline
        $\pi_{\mathrm{ash}}^*$ \cite{ashwood2022dynamic} & $-1.3546$& $-1.2778$\\ \hline
        $\pi^*_\mathrm{s}$ \cite{ziebart2008maximum}&$-1.3236$& $-1.3279$ \\ \hline
    \end{tabular}
    \caption{\small Performance of different policies in the transferability experiments. $\pi_1^*$ is the optimal policy of $r^{\mathrm{true}}$ in the Open GridWorld, and $\hat \pi_1^*$ is its sample estimate. In the novel environments, $\pi_{\mathrm{true}}^*$ is the optimal policy of $r^{\mathrm{true}}$, $\pi_r^*$ is the optimal policy of the learned reward from \eqref{eq:p2-approx}, $\pi_{\mathrm{ash}}^*$ is the optimal policy using \cite{ashwood2022dynamic}, and $\pi^*_\mathrm{s}$ is the optimal policy using \cite{ziebart2008maximum}.}
    \label{tab:nll_table}
    \vspace{-0.33cm}
\end{table}

\section{CONCLUSION}
In this work, we addressed the challenge of reward identification in finite-horizon, time-varying settings by introducing a unifying framework that incorporates sparsity and rank priors. Our approach efficiently recovers minimally switching rewards through a greedy interval partitioning algorithm and leverages low-rank matrix approximations to identify structured feature-based rewards. Empirical results on several gridworld environments demonstrate robustness to policy estimation noise and superior transferability compared to existing methods.

\bibliography{references}

\end{document}